\documentclass{article}
\usepackage{algorithmic}

\usepackage{hyperref}

\usepackage{amsmath}
\usepackage{amsthm}
\usepackage{amsfonts}
\usepackage{graphicx}
\usepackage{latexsym}
\usepackage{multicol}
\usepackage{fullpage}
\usepackage[T1]{fontenc}
\usepackage{amssymb}
\usepackage{mathrsfs}
\usepackage{color}
\usepackage{algorithmic, algorithm}
\newtheorem{definition}{Definition}
\newtheorem{theorem}{Theorem}

\newtheorem{proposition}{Proposition}
\newtheorem{lemma}{Lemma}

\usepackage{bbm}
\usepackage{makeidx}

\makeatother

\global\long\def\argmin{\operatorname*{arg\, min}}

\global\long\def\rank{\operatorname*{rank}}

\global\long\def\argmin{\operatorname*{arg\, min}}

\global\long\def\sign{\operatorname{sign}}

\global\long\def\prox{\operatorname{prox}}

\global\long\def\diag{\operatorname{diag}}

\newcommand{\RR}{\mathbb{R}}

\newcommand{\LL}{\mathcal{L}}

\newcommand{\trans}{^{\scriptscriptstyle \top}}

\newcommand{\norm}[1]{\|#1\|}
\newcommand{\bignorm}[1]{\Big\|#1\Big\|}
\newcommand{\bX}{\mathbf X}
\newcommand{\bN}{\mathbf N}

\newcommand{\inr}[1]{\langle #1 \rangle}

\newcommand{\R}{\mathbb R}

\newcommand{\lmax}{\lambda_{\max}}
\DeclareMathOperator{\tr}{tr}
\newcommand{\E}{\mathbb E}
\newcommand{\EE}{\mathcal E}
\newcommand{\cA}{\mathcal A}
\newcommand{\cF}{\mathcal F}
\newcommand{\cC}{\mathcal C}
\newcommand{\cL}{\mathcal L}
\newcommand{\cP}{\mathcal P}
\newcommand{\cW}{\mathcal W}
\newcommand{\mleq}{\preceq}

\renewcommand{\P}{\mathbb P}
\newcommand{\eps}{\varepsilon}

\newcommand{\op}{\mathrm{op}}

\newcommand{\normop}[1]{\norm{#1}_{\op}}

\newtheorem{assumption}{Assumption}

\usepackage{nips12submit_e,times}
\title{Link Prediction in Graphs with Autoregressive
  Features}

\author{
Emile Richard \\
CMLA UMR CNRS 8536, \\ ENS Cachan, France\\%& 1000mercis\\
\And
St{\'e}phane Ga{\"i}ffas \\
CMAP - Ecole Polytechnique \\ \& LSTA - Universit{\'e} Paris 6\\
\And
Nicolas Vayatis \\
CMLA UMR CNRS 8536, \\ ENS Cachan, France\\
}

\nipsfinalcopy % Uncomment for camera-ready version

\begin{document}

\maketitle

\begin{abstract}
In the paper, we consider the problem of link prediction in time-evolving graphs. We assume that certain graph features, such as the node degree, follow a vector autoregressive (VAR) model and we propose to use this information to improve the accuracy of prediction. Our strategy involves a joint optimization procedure over the space of adjacency matrices and VAR matrices which takes into account both sparsity and low rank properties of the matrices. Oracle inequalities are derived and illustrate the trade-offs in the choice of smoothing parameters when modeling the joint effect of sparsity and low rank property. The estimate is computed efficiently using proximal methods through a generalized forward-backward agorithm.
\end{abstract}

\section{Introduction}
% blabla general
Forecasting systems behavior with multiple responses has been a
challenging issue in many contexts of applications such as
collaborative filtering, financial markets, or bioinformatics, where
responses can be, respectively, movie ratings, stock prices, or
activity of genes within a cell. Statistical modeling techniques have
been widely investigated in the context of multivariate time series
either in the multiple linear regression setup \cite{Breiman97} or
with autoregressive models \cite{Tsay05}. More recently, kernel-based
regularized methods have been developed for multitask learning
\cite{Evgeniou05, Argyriou07}. These approaches share  the
use of the correlation structure among input variables to enrich the
prediction on every single output. Often, the correlation structure is
assumed to be given or it is estimated separately. A discrete encoding
of correlations between variables can be modeled as a graph so that
learning the dependence structure amounts to performing graph
inference through the discovery of uncovered edges on the graph. The
latter problem is interesting {\em per se} and it is known as the
problem of link prediction where it is assumed that only a part of the
graph is actually observed \cite{liben2007link, Kolar11}. This situation
occurs in various applications such as recommender systems, social
networks, or proteomics, and the appropriate tools can be found among
matrix completion techniques \cite{Srebro05, Candes09,
Abernethy09}. In the realistic setup of a time-evolving graph,
matrix completion was also used and adapted to take into account the
dynamics of the features of the graph \cite{Richard10}.
% appli web
In this paper, we study the prediction problem where the observation is a sequence of
graphs adjacency matrices $(A_t)_{0 \leq t \leq T}$ and the goal is to
predict $A_{T+1}$. This type of problem arises in applications such as
recommender systems where, given information on purchases made by some
users, one would like to predict future purchases. In this context,
users and products can be modeled as the nodes of a bipartite graph, while
purchases or clicks are modeled as edges.
% appli bio
In functional genomics and systems biology, estimating regulatory
networks in gene expression can be performed by modeling the data as
graphs and fitting predictive models is a natural way for estimating
evolving networks in these contexts.
% etat de l'art
A large variety of methods for link prediction only consider
predicting from a single static snapshot of the graph - this includes
heuristics \cite{liben2007link, sarkar2010theoretical}, matrix
factorization \cite{koren2008factorization}, diffusion \cite{Myers10}, or probabilistic methods
\cite{taskar2003link}. More recently, some works have investigated
using sequences of observations of the graph to improve the
prediction, such as using regression on features extracted from the
graphs \cite{Richard10}, using matrix factorization
\cite{koren2010collaborative}, continuous-time regression \cite{Vu11}.
% notre hypothese
Our main assumption is that the network effect is a cause and a symptom at the same time, and therefore, the edges and the graph features should be estimated simultaneously. We propose a regularized approach to predict the uncovered links and the evolution of the graph features simultaneously. We provide oracle bounds under the assumption that the noise sequence has subgaussian tails and we prove that our procedure achieves a trade-off in the calibration of smoothing parameters which adjust with the sparsity and the rank of the unknown adjacency matrix.
The rest of this paper is organized as follows. In Section 2, we describe the general setup of our work with the main assumptions and we formulate a regularized optimization problem which aims at jointly estimating the autoregression parameters and predicting the graph. In Section 3, we provide technical results with oracle inequalities and other theoretical guarantees on the joint estimation-prediction. Section 4 is devoted to the description of the numerical simulations which illustrate our approach. We also provide an efficient algorithm for solving the optimization problem and show empirical results. The proof of the theoretical results are provided as supplementary material in a separate document.

\section{Estimation of low-rank graphs with autoregressive features}

Our approach is based on the asumption that features can explain most
of the information contained in the graph, and that these features are
evolving with time. We make the following assumptions about the
sequence $(A_t)_{t \geq 0}$ of adjacency matrices of the graphs
sequence.

\paragraph*{Low-Rank.} We assume that the matrices $A_t$ have
low-rank. This reflects the presence of highly connected groups of
nodes such as communities in social networks, or product categories and groups of loyal/fanatic users in a market place data, and is sometimes motivated by the small number of factors that explain nodes interactions.

\paragraph*{Autoregressive linear features.}

We assume to be given a linear map $\omega : \RR^{n \times
  n}\rightarrow \RR^d$ defined by
\begin{equation}
  \label{eq:defOm} 
  \omega(A) = \Big( \inr{\Omega_1 , A}, \cdots ,\inr{\Omega_d , A}  \Big),
\end{equation}
where $(\Omega_i)_{1\leq i \leq d}$ is a set of $n \times n$
matrices. These matrices can be either deterministic or random in our
theoretical analysis, but we take them deterministic for the sake of
simplicity. The vector time series $(\omega(A_t))_{t \geq 0}$ has
autoregressive dynamics, given by a VAR (Vector Auto-Regressive)
model:
\begin{equation*}
  \omega(A_{t+1}) = W_0^\top \omega(A_t) + N_{t+1},
\end{equation*}
where $W_0 \in \RR^{d \times d}$ is a unknown sparse matrix and
$(N_t)_{t \geq 0}$ is a sequence of noise vectors in $\R^d$. An
example of linear features is the degree ({\it i.e.} number of edges
connected to each node, or the sum of their weights if the edges are
weighted), which is a measure of popularity in social and commerce
networks. Introducing
\begin{equation*}
  \bX_{T-1} = (\omega(A_0), \ldots, \omega(A_{T-1}))^\top \;\; \text{ and
  } \;\; \bX_{T} = (\omega(A_1), \ldots, \omega(A_{T}))^\top,
\end{equation*}
which are both $T \times d$ matrices, we can write this model in a
matrix form:
\begin{equation}
  \label{eq:matrix-autoregressive}
  \bX_T = \bX_{T-1} W_0 + \bN_T,
\end{equation}
where $\bN_T = (N_1, \ldots, N_T)^\top$.

This assumes that the noise is driven by time-series dynamics (a
martingale increment), where each coordinates are independent (meaning
that features are independently corrupted by noise), with a
sub-gaussian tail and variance uniformly bounded by a constant
$\sigma^2$. In particular, no independence assumption between the
$N_t$ is required here.

\paragraph*{Notations.}

The notations $\norm{\cdot}_F$, $\norm{\cdot}_p$,
$\norm{\cdot}_\infty$, $\norm{\cdot}_*$ and $\normop{\cdot}$ stand,
respectively, for the Frobenius norm, entry-wise $\ell_p$ norm,
entry-wise $\ell_\infty$ norm, trace-norm (or nuclear norm, given by
the sum of the singular values) and operator norm (the largest
singular value). We denote by $\inr{A, B} = \tr(A^\top B)$ the
Euclidean matrix product. A vector in $\R^d$ is always understood as a
$d \times 1$ matrix. We denote by $\norm{A}_0$ the number of non-zero
elements of $A$. The product $A \circ B$ between two matrices with
matching dimensions stands for the Hadamard or entry-wise product
between $A$ and $B$. The matrix $|A|$ contains the absolute values of entries of $A$. The matrix $(M)_+$ is the componentwise positive part of the matrix M, and $\textrm{sign}(M)$ is the sign matrix associated to $M$ with the convention $\textrm{sign}(0) = 0$

If $A$ is a $n \times n$ matrix with rank $r$, we write its SVD as $A
= U \Sigma V^\top = \sum_{j=1}^r \sigma_j u_j v_j^\top$ where $\Sigma
= \diag(\sigma_1, \ldots, \sigma_r)$ is a $r \times r$ diagonal matrix
containing the non-zero singular values of $A$ in decreasing order,
and $U = [u_1, \ldots, u_r]$, $V = [v_1, \ldots, v_r]$ are $n \times
r$ matrices with columns given by the left and right singular vectors
of $A$.  The projection matrix onto the space spanned by the
columns (resp. rows) of $A$ is given by $P_U = U U^\top$ (resp. $P_V =
V V^\top$). The operator $\cP_A : \R^{n \times n} \rightarrow \R^{n
  \times n}$ given by $\cP_A(B) = P_U B + B P_V - P_U B P_V$ is the
projector onto the linear space spanned by the matrices $u_k x^\top$
and $y v_k^\top$ for $1 \leq j, k \leq r$ and $x, y \in \R^{n}$. The
projector onto the orthogonal space is given by $\cP_A^\perp(B) = (I -
P_U) B (I - P_V)$. We also use the notation $a \vee b = \max(a, b)$.

\subsection{Joint prediction-estimation through penalized optimization}

In order to reflect the autoregressive dynamics of the features, we
use a least-squares goodness-of-fit criterion that encourages the
similarity between two feature vectors at successive time steps.  In
order to induce sparsity in the estimator of $W_0$, we penalize this
criterion using the $\ell_1$ norm. This leads to the following
penalized objective function:
\begin{equation*}
  J_1(W) = \frac{1}{d T} \norm{\bX_{T} - \bX_{T-1}W}_F^2 + \kappa \norm{W}_1,
\end{equation*}
where $\kappa > 0$ is a smoothing parameter. 

Now, for the prediction of $A_{T+1}$, we propose to minimize a
least-squares criterion penalized by the combination of an $\ell_1$
norm and a trace-norm. This mixture of norms induces sparsity and a
low-rank of the adjacency matrix. Such a combination of $\ell_1$ and
trace-norm was already studied in~\cite{6034724} for the matrix
regression model, and in~\cite{Richard12} for the prediction of an
adjacency matrix.

The objective function defined below exploits the fact that if $W$ is
close to $W_0$, then the features of the next graph $\omega(A_{T+1})$
should be close to $W^\top \omega(A_T)$. Therefore, we consider
\begin{equation*}
  J_2(A,W) = \frac{1}{d} \norm{\omega(A) - W^\top \omega(A_T)  }_F^2 +
  \tau \norm{A}_* + \gamma \norm{A}_1,
\end{equation*}
where $\tau, \gamma > 0$ are smoothing parameters. The overall 
objective function is the sum of the two partial objectives $J_1$ and
$J_2$, which is jointly convex with respect to $A$ and $W$:
\begin{equation}
  \label{eq:joint_objective}
  \LL (A,W) \doteq \frac{1}{d T} \norm{\bX_T - \bX_{T-1} W}_F^2 + \kappa
  \norm{W}_1 + \frac{1}{d} \norm{\omega(A) - W^\top \omega(A_{T})}_2^2 +
  \tau \norm{A}_* + \gamma \norm{A}_1,
\end{equation}
 If we choose convex cones $\cA \subset \R^{n \times n}$
and $\cW \subset \R^{d \times d}$, our joint estimation-prediction
procedure is defined by
\begin{equation}
  \label{eq:Shat_W_hat_def}
  (\hat A, \hat W) \in \argmin_{(A, W) \in \cA \times \cW} \cL(A, W).
\end{equation}
It is natural to take $\cW = \R^{d \times d}$ and $\cA = (\R_+)^{n
  \times n}$ since there is no {\it a priori} on the values of the feature
matrix $W_0$, while the entries of the matrix $A_{T+1}$ must be
positive.

In the next section we propose oracle inequalities which prove that
this procedure can estimate $W_0$ and predict $A_{T+1}$ at the same
time.

\subsection{Main result}
The central contribution of our work is to bound the prediction error with high probability under the following natural hypothesis on the noise process.
\begin{assumption}
  \label{ass:noise}
  We assume that $(N_t)_{t \geq 0}$ satisfies $\E [N_t | \cF_{t-1}] =
  0$ for any $t \geq 1$ and that there is $\sigma > 0$ such that for
  any $\lambda \in \R$ and $j=1, \ldots, d$ and $t \geq 0$\textup:
  \begin{equation*}
    \E [ e^{\lambda (N_t)_j} | \cF_{t-1} ] \leq e^{\sigma^2 \lambda^2
      / 2}.
  \end{equation*}
  Moreover, we assume that for each $t \geq 0$, the coordinates
  $(N_t)_1, \ldots, (N_t)_d$ are independent.
\end{assumption}
The main result can be summarized  as follows. The prediction error and the estimation error  can be simultaneously bounded by the sum of three terms that involve homogeneously (a) the sparsity, (b) the rank of the adjacency matrix $A_{T+1}$, and (c) the sparsity of the VAR model matrix $W_0$.  The tight bounds we obtain are similar to the bounds of the Lasso and are upper bounded by: 

$$ C_1 \sqrt{\frac{\log d}{Td^2}}\|W_0\|_0 + C_2\sqrt{\frac{\log n}{ d}}\|A_{T+1}\|_0 +C_3\sqrt{\frac{\log n}{ d}}\rank A_{T+1} ~~.$$
The positive constants $C_1,C_2,C_3$ are proportional to the noise level $\sigma$. The interplay between the rank and sparsity constraints on $A_{T+1}$ are reflected in the observation that the values of $C_2$ and $C_3$ can be changed as long as their sum remains constant. 

\section{Oracle inequalities}
\label{sec:oracle}

In this section we give oracle inequalities for the mixed
prediction-estimation error which is given, for any $A \in \R^{n
  \times n}$ and $W \in \R^{d \times d}$, by
\begin{equation}
  \label{eq:defE}
  \EE( {A},  {W})^2 \doteq \frac{1}{d} \| ( {W} - W_0)^\top
  \omega(A_{T}) - \omega( {A}-A_{T+1}) \|_2^2 +  \frac{1}{dT} \|
  \bX_{T-1}( {W } - W_0)\|_F^2.
\end{equation}
It is important to have in mind that an upper-bound on $\EE$ implies
upper-bounds on each of its two components. It entails in particular
an upper-bound on the feature estimation error $ \|
\bX_{T-1}(\widehat{W}-W_0)\|_F$ that makes $\|(\widehat{W}-W_0)^\top
\omega(A_{T}) \|_2$ smaller and consequently controls the prediction
error over the graph edges through $\| \omega(\widehat{A}-A_{T+1})
\|_2$.

The upper bounds on $\EE$ given below exhibit the dependence of the
accuracy of estimation and prediction on the number of features $d$,
the number of edges $n$ and the number $T$ of observed graphs in the
sequence. 

Let us recall $\bN_T = (N_1, \ldots, N_T)^\top$ and introduce the
noise processes
\begin{equation*}
  M = - \sum_{j=1}^d (N_{T+1})_j \Omega_j \;\; \text{ and } \;\;
  \Xi = \sum_{t=1}^{T+1} \omega(A_{t-1}) N_t^\top,
\end{equation*}
which are, respectively, $n \times n$ and $d \times d$ random
matrices. The source of randomness comes from the noise sequence
$(N_t)_{t \geq 0}$, see Assumption~\ref{ass:noise}. If these noise
processes are controlled correctly, we can prove the following oracle
inequalities for procedure~\eqref{eq:Shat_W_hat_def}. The next result
is an oracle inequality of slow type (see for
instance~\cite{Bickel09}), that holds in full generality.

\begin{theorem}
  \label{thm:slowrate}
  Let $(\hat A, \hat W)$ be given by~\eqref{eq:Shat_W_hat_def} and
  suppose that
  \begin{equation}
    \label{eq:slow-rate-smoothing-parameters}
    \tau \geq \frac{2\alpha}{d} \| M \|_{\op}, \quad \gamma \geq
    \frac{2 (1-\alpha
      )}{d} \| M \|_\infty \;\; \text{ and } \;\; \kappa \geq
    \frac{2}{d T}\|\Xi\|_\infty
  \end{equation}
  for some $\alpha \in (0, 1)$. Then, we have
  \begin{equation*}
    \label{eq:slowrate}
    \EE(\widehat{A}, \widehat{W})^2 \leq \inf_{(A, W) \in \cA \times \cW}
    \Big\{ \EE(A, W)^2 + 2 \tau \| A \|_* +
    2 \gamma \| A \|_1 +2 \kappa\| W \|_1 \Big\}.
  \end{equation*}
\end{theorem}

For the proof of oracle inequalities of fast type, the
\emph{restricted eigenvalue} (RE) condition introduced
in~\cite{Bickel09} and~\cite{MR2555200, MR2500227} is of importance.
Restricted eigenvalue conditions are implied by, and in general weaker
than, the so-called \emph{incoherence} or RIP (Restricted isometry
property, \cite{Candes05}) assumptions, which excludes, for instance,
strong correlations between covariates in a linear regression
model. This condition is acknowledged to be one of the weakest to
derive fast rates for the Lasso (see~\cite{MR2576316} for a comparison
of conditions).

Matrix version of these assumptions are introduced in
\cite{Koltchinskii11}. Below is a version of the RE assumption that
fits in our context. First, we need to introduce the two restriction
cones.

The first cone is related to the $\norm{W}_1$ term used in
procedure~\eqref{eq:Shat_W_hat_def}. If $W \in \R^{d \times d}$, we
denote by $\Theta_W = \sign(W) \in \{ 0, \pm 1 \}^{d \times d}$ the
signed sparsity pattern of $W$ and by $\Theta_W^\perp \in \{ 0, 1
\}^{d \times d}$ the orthogonal sparsity pattern. For a fixed matrix
$W \in \R^{d \times d}$ and $c > 0$, we introduce the cone
\begin{equation*}
  \cC_1(W, c) \doteq \Big\{ W' \in \mathcal{W} :
  \|\Theta_W^\perp \circ W' \|_1 \leq c \|\Theta_W \circ W' \|_1
  \Big\}.
\end{equation*}
This cone contains the matrices $W'$ that have their largest entries
in the sparsity pattern of $W$.

The second cone is related to mixture of the terms $\norm{A}_*$ and
$\norm{A}_1$ in procedure~\eqref{eq:Shat_W_hat_def}. Before defining
it, we need further notations and definitions. 

For a fixed $A \in \R^{n \times n}$ and $c, \beta > 0$, we introduce
the cone
\begin{equation*}
  \cC_2(A, c, \beta) \doteq \Big\{ A' \in \cA : \| \cP_{A}^\perp(A') \|_*
  + \beta \|\Theta_A^\perp \circ A'\|_1
  \leq c \Big(  \| \mathcal{P}_{A}(A') \|_* + \beta \|\Theta_A \circ
  A' \|_1 \Big)  \Big\}.
\end{equation*}
This cone consist of the matrices $A'$ with large entries close to
that of $A$ and that are ``almost aligned'' with the row and column
spaces of $A$. The parameter $\beta$ quantifies the interplay between
these too notions.

% \begin{definition}[Cone of restriction $\mathcal{C}_{S,c,\beta}$]
% For a matrix $A \in \mathcal{A} \subset \RR^{n \times n}$ of rank
% $r$, and sparsity $k$, we define the cone of restriction, that
% contains matrices that projection onto the singular spaces of  $A$
% dominate the orthogonal projection and such that the projection onto
% the sparsity pattern dominates the orthogonal with a trade-off ratio
% $\beta$ between the two constraints. Let $A = U \Sigma V\trans$ be
% the SVD of $A$,  $U^\perp$ and $V^{\perp}$ matrices of size $n
% \times (n-r)$ ortho-normally completing the bases of $U$ and $V$,
% and for any matrix $B \in \RR^{n\times n}$ define the orthogonal
% projections $\mathcal{P}_{A}^\perp(B) = P_{U^\perp} B P_{V^\perp}$
% and $\mathcal{P}_{A}(B) = B -  \mathcal{P}_{A}^\perp(B)$. Let
% $\Theta_A = \sign(S)$ be the signed sparsity pattern of $A$, and
% $\Theta_A^\perp \in \{0,1 \}^{n\times n}$ the orthogonal sparsity
% pattern.\\
 
\begin{definition}[Restricted Eigenvalue (RE)]
  \label{ass:RE}
  For $W \in \cW$ and $c > 0$, we introduce
  \begin{equation*}
    \mu_1(W, c) = \inf \Big\{ \mu > 0 : \norm{\Theta_W \circ W'}_F \leq
    \frac{\mu}{\sqrt{d T}} \norm{\bX_{T+1} W'}_F, \;\; \forall W' \in
    \cC_1(W, c) \Big\}.
  \end{equation*}
  For $A \in \cA$ and $c, \beta > 0$, we introduce
  \begin{align*}
    \mu_2(A, W, c, \beta) = \inf \Big\{ \mu > 0 :\;&\norm{\cP_A(A')}_F
    \vee
    \norm{\Theta_A \circ A'}_F \\
    &\leq \frac{\mu}{\sqrt d} \norm{W'^\top \omega(A_T) -
      \omega(A')}_2, \;\; \forall W' \in \cC_1(W, c), \forall A' \in
    \cC_2(A, c, \beta) \Big\}.
  \end{align*}
\end{definition}
The RE assumption consists of assuming that the constants $\mu_1$ and $\mu_2$ are non-zero. Now we can state the following Theorem that gives a fast oracle
inequality for our procedure using RE.

\begin{theorem}
  \label{prop:fastrateRE}
  Let $(\hat A, \hat W)$ be given by~\eqref{eq:Shat_W_hat_def} and
  suppose that
  \begin{equation}
    \label{eq:thm-fast-rate-smoothing-parameters}
    \tau \geq \frac{3\alpha}{d} \| M \|_{\op}, \quad \gamma \geq
    \frac{3 (1-\alpha)}{d} \| M \|_\infty \;\; \text{ and } \;\;
    \kappa \geq \frac{3}{d T}\|\Xi\|_\infty
  \end{equation}
  for some $\alpha \in (0, 1)$. Then, we have
  \begin{align*}
    % \label{eq:fastrate}
    \EE(\widehat{A}, \widehat{W})^2 \leq \inf_{(A, W) \in \cA \times
      \cW} \Big\{ \EE(A, W)^2 &+ \frac{25}{18} \mu_2(A, W)^2 \big(
    \rank(A) \tau^2 + \norm{A}_0 \gamma^2) \\
    &+ \frac{25}{36} \mu_1(W)^2 \norm{W}_0 \kappa^2 \Big\},
  \end{align*}
  where $\mu_1(W) = \mu_1(W, 5)$ and $\mu_2(A, W) = \mu_2(A, W, 5,
  \gamma / \tau)$ (see Definition~\ref{ass:RE}).
\end{theorem}
The proofs of Theorems~\ref{thm:slowrate}
and~\ref{prop:fastrateRE} use tools introduced in
\cite{Koltchinskii11} and \cite{Bickel09}.

Note that the residual term from this oracle inequality mixes the
notions of sparsity of $A$ and $W$ via the terms $\rank(A)$,
$\norm{A}_0$ and $\norm{W}_0$. It says that our mixed penalization
procedure provides an optimal trade-off between fitting the data and
complexity, measured by both sparsity and low-rank. This is the first
result of this nature to be found in literature.

In the next Theorem~\ref{thm:convergence-rates}, we obtain convergence
rates for the procedure~\eqref{eq:Shat_W_hat_def} by combining
Theorem~\ref{prop:fastrateRE} with controls on the noise processes. We introduce
\begin{align*}
  v_{\Omega, \op}^2 &= \bignorm{\frac 1d \sum_{j=1}^d \Omega_j^\top
    \Omega_j}_{\op} \vee \bignorm{\frac 1d \sum_{j=1}^d \Omega_j
    \Omega_j^\top}_{\op}, \quad v_{\Omega, \infty}^2 = \bignorm{\frac
    1d \sum_{j=1}^d \Omega_j \circ \Omega_j}_{\infty}, \\
  \sigma_\omega^2 &= \max_{j=1, \ldots, d} \frac{1}{T+1}
  \sum_{t=1}^{T+1} \omega_j(A_{t-1})^2,
\end{align*}
which are the (observable) variance terms that naturally appear in the
controls of the noise processes. We introduce also
\begin{equation*}
  \ell_T = 2 \max_{j = 1, \ldots, d} \log \log \bigg(
  \frac{\sum_{t=1}^{T+1} \omega_j(A_{t-1})^2}{T+1} \vee
  \frac{T+1}{\sum_{t=1}^{T+1} \omega_j(A_{t-1})^2} \vee e \bigg),
\end{equation*}
which is a small (observable) technical term that comes out of our
analysis of the noise process $\Xi$. This term is a small price to pay
for the fact that no independence assumption is required on the noise
sequence $(N_t)_{t \geq 0}$, but only a martingale increment structure
with sub-gaussian tails.

\begin{theorem}
  \label{thm:convergence-rates}
  Consider the procedure $(\hat A, \hat W)$ given
  by~\eqref{eq:Shat_W_hat_def} with smoothing parameters given by
  \begin{align*}
    \tau &= 3 \alpha \sigma v_{\Omega, \op} \sqrt{\frac{2(x +
        \log(2n))}{d}}, \quad \gamma = 3 (1-\alpha) \sigma v_{\Omega,
      \infty} \sqrt{\frac{2(x + 2 \log n)}{d}}, \\
    \kappa &= 6 \sigma \sigma_\omega \frac 1d \sqrt{\frac{2 e (x + 2
        \log d + \ell_T)}{T+1}}
  \end{align*}
  for some $\alpha \in (0, 1)$ and fix a confidence level $x >
  0$. Then, we have
  \begin{align*}
    % \label{eq:fastrate}
    \EE(\widehat{A}, \widehat{W})^2 \leq \inf_{(A, W) \in \cA \times
      \cW} \Big\{ \EE(A, W)^2 &+ 25 \mu_2(A)^2 \rank(A) \alpha^2
    \sigma^2 v_{\Omega, \op}^2 \frac{2(x + \log(2n))}{d} \\
    &+ 25 \mu_2(A)^2 \norm{A}_0 (1-\alpha)^2 \sigma^2 v_{\Omega,
      \infty}^2 \frac{2(x + 2 \log n)}{d} \\
    &+ 25 \mu_1(W)^2 \norm{W}_0 \sigma^2 \sigma_\omega^2 \frac{2 e (x
      + 2 \log d + \ell_T)}{d^2(T+1)} \Big\}
  \end{align*}  
  with a probability larger than $1 - 17 e^{-x}$, where $\mu_1$ and
  $\mu_2$ are the same as in Theorem~\ref{prop:fastrateRE}. 
\end{theorem}

The proof of Theorem~\ref{thm:convergence-rates} follows directly from
Theorem~\ref{prop:fastrateRE} basic noise control results. In the next Theorem, we propose
more explicit upper bounds for both the indivivual estimation of $W_0$
and the prediction of $A_{T+1}$.
\begin{theorem}
  \label{thm:take-away-message}
  Under the same assumptions as in
  Theorem~\ref{thm:convergence-rates}, for any $x > 0$
  the following inequalities hold with a probability larger than $1 -
  17e^{-x}$:
  \begin{equation}
    \label{eq:explicit-upper-bound-W-prediction}
    \begin{split}
      \frac{1}{d T}& \norm{\bX_T (\hat W - W_0)}_F^2 \\
      &\leq \inf_{A \in
        \cA} \Big\{ \frac 1d \norm{\omega(A) - \omega(A_{T+1})}_F^2 +
      \frac{25}{18} \mu_2(A, W)^2 \big(
      \rank(A) \tau^2 + \norm{A}_0 \gamma^2) \Big\} \\
      & \quad + \frac{25}{36} \mu_1(W_0)^2 \norm{W_0}_0 \kappa^2
    \end{split}
  \end{equation}
 
  \begin{equation}
    \label{eq:explicit-upper-bound-W-l1}
    \begin{split}
     & \norm{\hat W - W_0}_1 \leq  5 \mu_1(W_0)^2 \norm{W_0}_0 \kappa \\
      &+  6 \sqrt{\norm{W_0}_0} \mu_1(W_0) \inf_{A \in \cA} \sqrt{ \frac 1d \norm{\omega(A) - \omega(A_{T+1})}_F^2 +  \frac{25}{18} \mu_2(A, W)^2 \big(\rank(A) \tau^2 + \norm{A}_0 \gamma^2) } \\
      \end{split}
  \end{equation}

  \begin{equation}
    \label{eq:explicit-upper-bound-A-trace}    
    \begin{split}
      &\norm{\hat A - A_{T+1}}_* \leq5 \mu_1(W_0)^2 \norm{W_0}_0 \kappa+ (6 \sqrt{\rank A_{T+1}} + 5 \beta \sqrt{\norm{A_{T+1}}_0})
      \mu_2( A_{T+1}) \\
      &\quad \quad \quad \times \inf_{A \in \cA} \sqrt{ \frac 1d
        \norm{\omega(A) - \omega(A_{T+1})}_F^2 + \frac{25}{18}
        \mu_2(A, W)^2 \big( \rank(A) \tau^2 + \norm{A}_0 \gamma^2) }~.
      \\
       \end{split}
  \end{equation}
  
\end{theorem}

\section{Algorithms and Numerical Experiments}

\subsection{Generalized forward-backward algorithm for minimizing $\LL$}
We use the algorithm designed in \cite{raguet2011generalized} for minimizing our objective function. Note that this algorithm is preferable to the method introduced in \cite{Richard10} as it directly minimizes $\LL$ jointly in $(S,W)$ rather than alternately minimizing in $W$ and $S$. 

Moreover we use the novel joint penalty from \cite{Richard12} that is more suited for estimating graphs.
 The proximal operator for the trace norm is given by the shrinkage operation, if $Z = U \diag (\sigma_1, \cdots, \sigma_n) V^T$ is the singular value decomposition of $Z$, 
$$ \prox_{\tau ||.||_*}(Z) = U \diag ((\sigma_i - \tau)_+)_i V^T.$$
Similarly, the proximal operator for the $\ell_1$-norm is the soft thresholding operator defined by using the entry-wise product of matrices denoted by $\circ$: 
$$\prox_{\gamma ||.||_1} = \textrm{sgn}(Z) \circ (|Z| - \gamma)_+\,.$$
The algorithm converges under very mild conditions when the step size $\theta$ is smaller than $\frac{2}{L}$, where $L$ is the operator norm of the joint quadratic loss: 
$$\Phi : (A,W) \mapsto     \frac{1}{d T} \norm{\bX_T - \bX_{T-1} W}_F^2 + \frac{1}{d} \norm{\omega(A) - W^\top \omega(A_T)  }_F^2~~.$$

\begin{algorithm}[tbh]
   \caption{Generalized Forward-Backward to Minimize $\LL$}
   \label{alg:gfb}
\begin{algorithmic}
   \STATE Initialize $A, Z_1, Z_2, W, q = 2$
   %\FOR{$i=1$ {\bfseries to} $n_{\textrm{iterations}}$}
   \REPEAT
   \STATE Compute $(G_A,G_W) = \nabla_{A,W} \Phi (A, W)$.
   \STATE Compute $Z_1 = \prox_{q \theta \tau ||.||_{*}} (2A - Z_1 - \theta G_A)$
   \STATE Compute $Z_2 = \prox_{q \theta \gamma ||.||_{1}} (2A - Z_2 - \theta G_A)$
   \STATE Set $A = \frac{1}{q} \sum_{k=1}^q Z_k$
      \STATE Set $W = \prox_{\theta \kappa ||.||_{1}} (W - \theta G_W)$
   \UNTIL{convergence}
   \STATE {\bfseries return} $(A,W)$ minimizing $\LL$
\end{algorithmic}
\end{algorithm}

\subsection{A generative model for graphs having linearly
  autoregressive features}\label{sec:gen}

Let $V_0 \in \RR^{n \times r}$ be a sparse matrix, $V_0 ^\dagger $ its
pseudo-inverse such, that $V_0 ^\dagger V_0 = V_0\trans V_0
^{{\scriptscriptstyle \top} \dagger} = I_r$. Fix two sparse matrices
$W_0 \in \RR^{r \times r}$ and $U_0 \in \RR^{n \times r}$ .  Now
define the sequence of matrices $(A_t)_{t \geq 0}$ for $t = 1, 2,
\cdots $ by
\[
U_t = U_{t-1}W_0 + N_t\]
and
\[
A_t = U_tV_0\trans + M_t
\]
for i.i.d sparse noise matrices $N_t$ and $M_t$, which means that for any pair of indices $(i,j)$, with high probability $(N_t)_{i,j}=0$ and $(M_t)_{i,j}=0$. We define the linear feature map $\omega(A) = A V_0^{{\scriptscriptstyle \top} \dagger} $, and point out that 
\begin{enumerate}
\item The sequence $\bigg (\omega(A_t) \trans \bigg )_t = \bigg (U_t + M_t V_0^{{\scriptscriptstyle \top} \dagger} \bigg )_t$ follows the linear autoregressive relation \[ \omega(A_t) \trans = \omega(A_{t-1}) \trans W_0 + N_t +  M_t V_0^{{\scriptscriptstyle \top} \dagger} ~~.\]
\item For any time index $t$, the matrix $A_t$ is close to $U_tV_0$ that has rank at most $r$
\item The matrices $A_t$ and $U_t$ are both sparse by construction.
\end{enumerate}

\subsection{Empirical evaluation}
We tested the presented methods on synthetic data generated as in section (\ref{sec:gen}). In our experiments the noise matrices $M_t$ and $N_t$ where built by soft-thresholding {\it i.i.d.} noise $\mathcal{N}(0,\sigma^2)$. We took as input $T=10$  successive graph snapshots on $n=50$ nodes graphs of rank $r=5$. We used $d=10$ linear features, and finally the noise level was set to $\sigma = .5$. We compare our methods to standard baselines in link prediction. We use the area under the ROC curve as the measure of performance and report empirical results averaged over 50 runs with the corresponding confidence intervals in figure \ref{fig:perf}. The competitor methods are the {\it nearest neighbors} (NN) and static sparse and low-rank estimation, that is the link prediction algorithm suggested in \cite{Richard12}. The algorithm NN scores pairs of nodes with the number of common friends between them, which is given by $A^2$ when $A$ is the cumulative graph adjacency matrix $\widetilde{A_T} = \sum_{t=0}^TA_t$ and the static sparse and low-rank estimation is obtained by minimizing the objective $\|X-\widetilde{A_T}\|_F^2+ \tau\|X\|_*+\gamma\|X\|_1$, and can be seen as the closest {\it static} version of our method. The two methods {\it autoregressive low-rank} and {\it static low-rank} are regularized using only the trace-norm, (corresponding to forcing $\gamma=0$) and are slightly inferior to their sparse and low-rank rivals. Since the matrix $V_0$ defining the linear map $\omega$ is unknown we consider the feature map $\omega(A) = AV$ where $\widetilde{A_T} = U\Sigma V\trans$ is the SVD of $\widetilde{A_T}$. The parameters $\tau$ and $\gamma$ are chosen by 10-fold cross validation for each of the methods separately. 

\begin{figure}
\begin{center}
\includegraphics[width = 6cm]{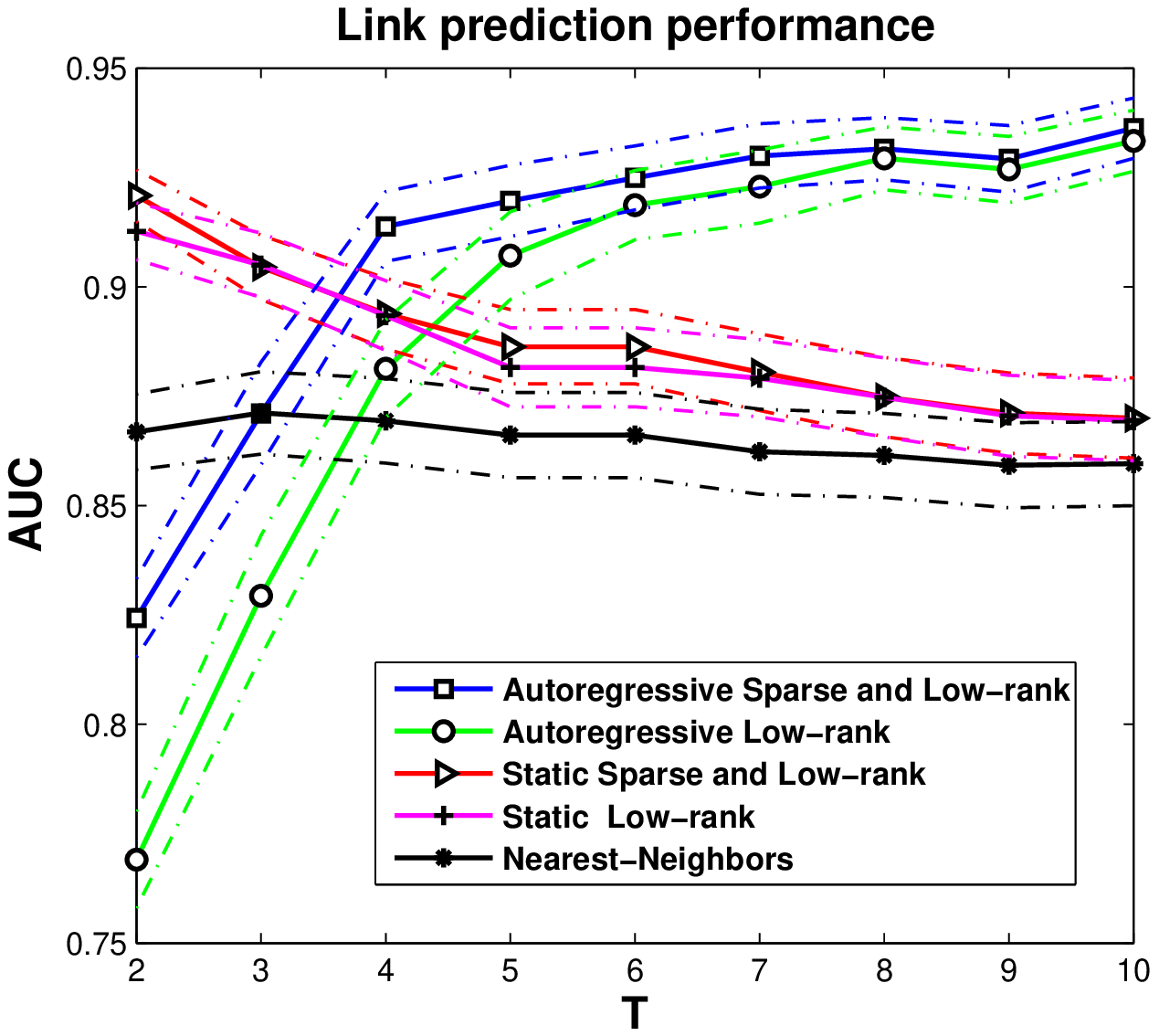}\includegraphics[width = 7cm]{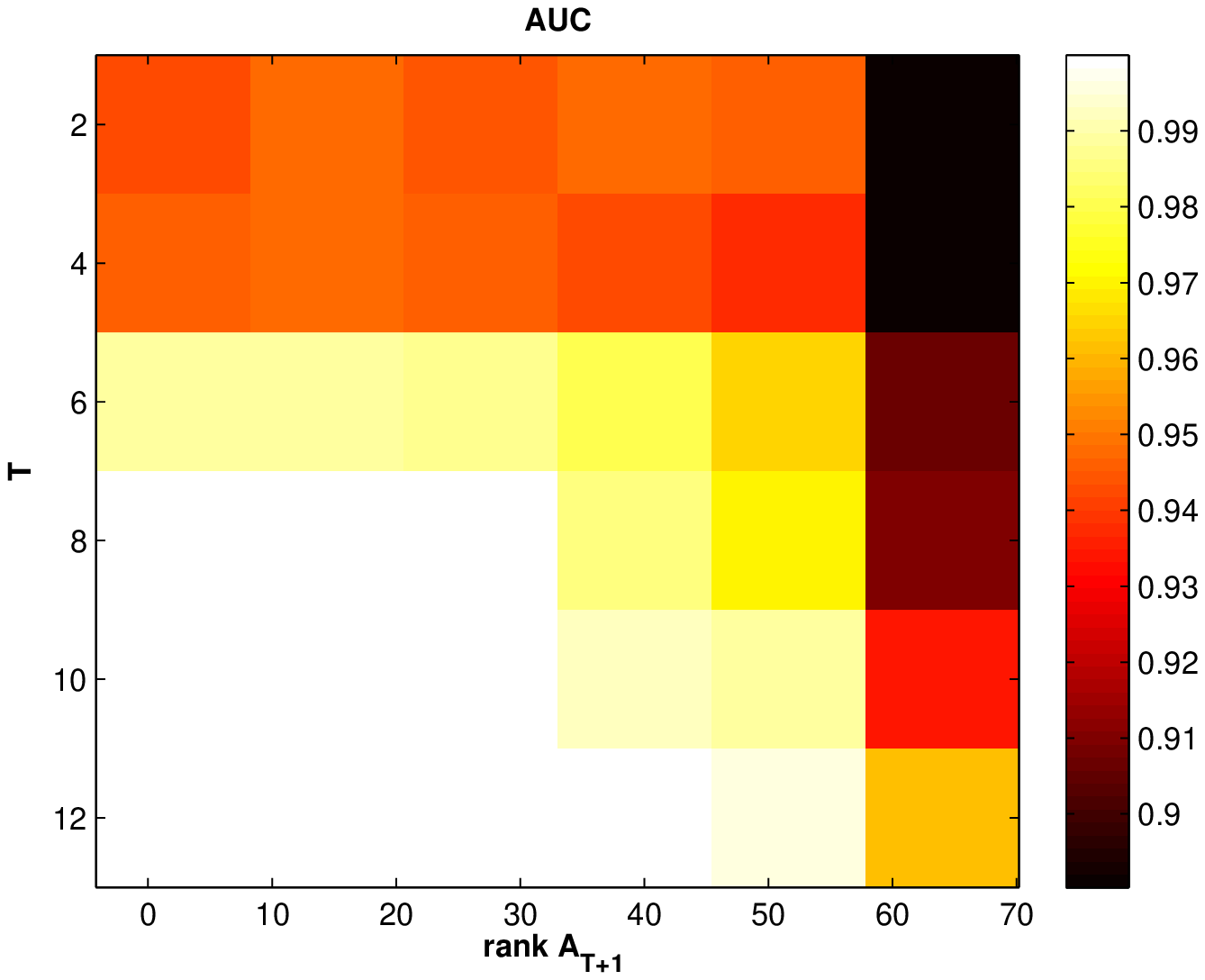}
\caption{Left: performance of algorithms in terms of Area Under the ROC Curve, average and confidence intervals over 50 runs. Right: Phase transition diagram. }
\end{center}
\label{fig:perf}
\end{figure}

\subsection{Discussion}

\begin{enumerate}
\item  {\it Comparison with the baselines.} This experiment sharply shows the benefit of using a temporal approach when one can handle the feature extraction task. The left-hand plot shows that if few snapshots are available ($T\leq 4$ in these experiments), then static approaches are to be preferred, whereas feature autoregressive approaches outperform as soon as {\it sufficient number} $T$ graph snapshots are available (see phase transition). The decreasing performance of static algorithms can be explained by the fact that they use as input a mixture of graphs observed at different time steps. Knowing that at each time step the nodes have specific latent factors, despite the slow evolution of the factors, adding the resulting graphs leads to confuse the factors. 
\item {\it Phase transition.} The right-hand figure is a phase transition diagram showing in which part of rank and time domain the estimation is accurate and illustrates the interplay between these two domain parameters.

 \item  {\it Choice of the feature map $\omega$.} In the current work we used the projection onto the vector space of the top-$r$ singular vectors of the cumulative adjacency matrix as the linear map $\omega$, and this choice has shown empirical superiority to other choices. The question of choosing the best measurement to summarize graph information as in compress sensing seems to have both theoretical and application potential.
Moreover, a deeper understanding of the connections of our problem with compressed sensing, for the construction and theoretical validation of the features mapping, is an important point that needs several developments. One possible approach is based on multi-kernel learning, that should be considered in a future work.
\item {\it Generalization of the method.} In this paper we consider only an autoregressive process of order 1. For better prediction accuracy, one could consider mode general models, such as vector ARMA models, and use model-selection techniques for the choice of the orders of the model. A general modelling based on state-space model could be developed as well. 
We presented a procedure for predicting graphs having linear autoregressive features. Our approach can easily be generalized to non-linear prediction through kernel-based methods.

\end{enumerate}

\appendix[Appendix : Proof of propositions]

\section{Proofs of the main results}
\label{sec:proofs}

From now on, we use the notation $\norm{(A, a)}_F^2 = \norm{A}_F^2 +
\norm{a}_2^2$ and $\inr{(A, a), (B, b)} = \inr{A, B} + \inr{a, b}$ for
any $A, B \in \R^{T \times d}$ and $a, b \in \R^d$.

Let us introduce the linear mapping $\Phi : \R^{n \times n} \times
\R^{d \times d} \rightarrow \R^{T \times d} \times \R^d$ given by
\begin{equation*}
  \Phi(A,W) = \Big( \frac{1}{\sqrt T} \bX_{T-1} W, \omega(A) - W^\top
  \omega(A_T) \Big).
\end{equation*}
Using this mapping, the objective~\eqref{eq:joint_objective} can be
written in the following reduced way:
\begin{equation*}
  \LL(A,W) = \frac{1}{d} \Big\|\Big(\frac{1}{\sqrt T} \bX_{T}, 0\Big)
  - \Phi(A,W) \Big\|_F^2 + \gamma
  \norm{A}_1 + \tau \norm{A}_* + \kappa \norm{W}_1.
\end{equation*}
Recalling that the error writes, for any $A$ and $W$:
\begin{equation*}
  \EE(A, W)^2 = \frac{1}{d} \norm{(W - W_0)^\top \omega(A_T) -
    \omega(A - A_{T+1})}_F^2 + \frac{1}{dT} \norm{\bX_{T-1} (W -
    W_0)}_F^2,
\end{equation*}
we have
\begin{equation*}
  \EE(A, W)^2 = \frac{1}{d} \big\| \Phi( A - A_{T+1}, W - W_0) \|_F^2.
\end{equation*}
Let us introduce also the empirical risk
\begin{equation*}
  R_n(A, W) = \frac{1}{d} \bignorm{\Big( \frac{1}{\sqrt T} \bX_T,
    0\Big) - \Phi(A, W)}_F^2.
\end{equation*}
The proofs of Theorem~\ref{thm:slowrate} and~\ref{prop:fastrateRE} are
based on tools developped in~\cite{Koltchinskii11} and
\cite{Bickel09}. However, the context considered here is very
different from the setting considered in these papers, so our proofs
require a different scheme.

\subsection{Proof of Theorem~\ref{thm:slowrate}}

First, note that 
\begin{align*}
  R_n(\hat A, \hat W) &- R_n(A, W) \\
  &= \frac 1d \Big( \norm{\Phi(\hat A, \hat W)}_F^2 - \norm{\Phi(A,
    W)}_F^2 - 2 \inr{ (\frac{1}{\sqrt T}\bX_T, 0), \Phi(\hat A - A,
    \hat W- W)} \Big).
\end{align*}
Since
\begin{align*}
  \frac 1d \Big(\norm{\Phi(\hat A, \hat W)}_F^2 &- \norm{\Phi(A,
    W)}_F^2 \Big) \\
  &= \EE(\hat A, \hat W)^2 - \EE(A, W)^2 + \frac 2d \inr{\Phi(\hat A -
    A, \hat W - W), \Phi(A_{T+1}, W_0)},
\end{align*}
we have
\begin{align*}
  R_n(\hat A, \hat W) &- R_n(A, W) \\
  &= \EE(\hat A, \hat W)^2 - \EE(A, W)^2 + \frac 2d \inr{\Phi(\hat A -
    A, \hat W - W), \Phi(A_{T+1}, W_0) - (\frac{1}{\sqrt T} \bX_T, 0)} \\
  &= \EE(\hat A, \hat W)^2 - \EE(A, W)^2 + \frac 2d \inr{\Phi(\hat A -
    A, \hat W - W), (-\frac{1}{\sqrt T}\bN_T, N_{T+1})}.
\end{align*}
The next Lemma will come in handy several times in the proofs.
\begin{lemma}
  \label{lem:noise-decomposition1}
  For any $A \in \R^{n \times n}$ and $W \in \R^{d \times d}$ we have
  \begin{equation*}
    \inr{ (\frac{1}{\sqrt T} \bN_{T}, -N_{T+1}), \Phi(A,W)} =
    \inr{(M, \frac 1T \Xi) , (A,W)} =
    \frac 1T \inr{W, \Xi} + \inr{A, M}.
  \end{equation*}
\end{lemma}
This Lemma follows from a direct computation, and the proof is thus
omitted.  This Lemma entails, together with~\eqref{eq:Shat_W_hat_def},
that
\begin{align*}
  \EE(\hat A, \hat W)^2 &\leq \EE(A, W)^2 + \frac {2}{dT} \inr{\hat W
    - W, \Xi} + \frac 2d \inr{\hat A - A, M} \\
  &\quad +\tau ( \|A\|_* - \|\widehat{A}\|_*) + \gamma ( \|A\|_1 - \|
  \widehat{A} \|_1 ) + \kappa ( \|W\|_1 - \|\widehat{W}\|_1).
\end{align*}
Now, using H\"older's inequality and the triangle inequality, and
introducing $\alpha \in (0, 1)$, we obtain
\begin{align*}
  \EE(\hat A, \hat W)^2 \leq \EE(A, W)^2 &+ \Big( \frac{2 \alpha}{d}
  \norm{M}_\op - \tau \Big) \norm{\hat A}_* + \Big( \frac{2 \alpha}{d}
  \norm{M}_\op + \tau \Big) \norm{A}_* \\
  &+ \Big( \frac{2 (1-\alpha)}{d} \norm{M}_\infty - \gamma \Big)
  \norm{\hat A}_1 + \Big( \frac{2 (1-\alpha)}{d} \norm{M}_\infty +
  \gamma \Big) \norm{A}_1 \\
  &+ \Big( \frac{2}{d T} \norm{\Xi}_\infty - \kappa \Big) \norm{\hat
    W}_1 + \Big( \frac{2}{d T} \norm{\Xi}_\infty + \kappa \Big)
  \norm{W}_1,
\end{align*}
which concludes the proof of Theorem~\ref{thm:slowrate},
using~\eqref{eq:slow-rate-smoothing-parameters}. $\hfill \square$

\subsection{Proof of Theorem~\ref{prop:fastrateRE}}

Let $A \in \R^{n \times n}$ and $W \in \R^{d \times d}$ be fixed, and
let $A = U \diag(\sigma_1, \ldots , \sigma_r) V\trans$ be the SVD of
$A$. Recalling that $\circ$ is the entry-wise product, we have $A =
\Theta_A \circ |A| + \Theta_A^\perp \circ A$, where $\Theta_A \in \{0,
\pm 1\}^{n \times n}$ is the entry-wise sign matrix of $A$ and
$\Theta_A^\perp \in \{ 0, 1 \}^{n \times n}$ is the orthogonal
sparsity pattern of $A$. 

The definition~\eqref{eq:Shat_W_hat_def} of $(\hat A, \hat W)$ is
equivalent to the fact that one can find $\hat G \in \partial \cL(\hat
A, \hat W)$ (an element of the subgradient of $\cL$ at $(\hat A, \hat
W)$) that belongs to the normal cone of $\cA \times \cW$ at $(\hat A,
\hat W)$. This means that for such a $\hat G$, and any $A \in \cA$ and
$W \in \cW$, we have
\begin{equation}
  \label{eq:argmin-caracterization}
  \inr{\hat G, (\hat A - A, \hat W - W)} \leq 0.
\end{equation}
Any subgradient of the function $g(A) = \tau \|A\|_* + \gamma\|A\|_1$
writes
\begin{equation*}
  Z = \tau Z_* + \gamma Z_1 = \tau \Big( U V^\top + \cP_A^\perp(G_*)
  \Big) + \gamma \Big( \Theta_A + G_1 \circ \Theta_A^\perp \Big)
\end{equation*}
for some $\|G_*\|_{\op} \leq 1$ and $\|G_1\|_\infty \leq 1$ (see for
instance \cite{MR1363368}). So, if $\hat Z \in \partial g(\hat A)$, we
have, by monotonicity of the sub-differential, that for any $Z
\in \partial g(A)$
\begin{equation*}
  \inr{\hat Z, \hat A - A} = \inr{\hat Z - Z, \hat A - A} + \inr{Z,
    \hat A - A} \geq \inr{Z, \hat A - A},
\end{equation*}
and, by duality, we can find $Z$ such that
\begin{equation*}
  \langle Z, \widehat{A} - A \rangle = \tau \langle U V^\top ,
  \widehat{A} - A \rangle + \tau \| \cP_A^\perp(\widehat{A}) \|_* + \gamma
  \langle \Theta_A, \widehat{A} - A \rangle + \gamma\| \Theta_A^\perp
  \circ \widehat{A}\|_1.
\end{equation*}
By using the same argument with the function $W \mapsto \norm{W}_1$
and by computing the gradient of the empirical risk $(A, W) \mapsto
R_n(A, W)$, Equation~\eqref{eq:argmin-caracterization} entails that
\begin{equation}
  \label{eq:thm2-main-decomposition}
  \begin{split}
    \frac{2}{d} \langle \Phi( \widehat{A} &- A_{T+1} , \widehat{W} -
    W_0 ), \Phi ( \widehat{A} - A , \widehat{W}-W ) \rangle \\
    &\leq \frac{2}{d} \langle (\frac{1}{\sqrt T} \bN_{T}, -N_{T+1}) ,
    \Phi( \widehat{A} - A , \widehat{W} - W ) \rangle - \tau \langle U
    V^\top, \widehat{A} - A \rangle - \tau \| \mathcal{P}_{A}^\perp(
    \widehat{A})\|_* \\
    &\quad - \gamma \langle \Theta_A, \widehat{A} - A \rangle -
    \gamma\| \Theta_A^\perp \circ \widehat{A}\|_1 - \kappa \langle
    \Theta_W, \widehat{W} - W \rangle - \kappa\| \Theta_W^\perp \circ
    \widehat{W}\|_1.
  \end{split}
\end{equation}
Using Pythagora's theorem, we have
\begin{equation}
  \label{eq:pythagoras}
  \begin{split}
    2 \langle \Phi &( \widehat{A} - A_{T+1}, \hat W - W_0 ) , \Phi (
    \widehat{A} - A , \widehat{W}-W) \rangle \\
    & = \| \Phi ( \widehat{A} - A_{T+1}, \hat W - W_0 ) \|_2^2 + \|
    \Phi ( \widehat{A} - A , \widehat{W}-W) \|_2^2 - \| \Phi ( A -
    A_{T+1}, W - W_0 ) \|_2^2.
  \end{split}
\end{equation}
It shows that if $\langle \Phi ( \widehat{A} - A_{T+1}, W-W_0 ) , \Phi
( \widehat{A} - A , \widehat{W}-W) \rangle \leq 0$, then
Theorem~\ref{prop:fastrateRE} trivially holds. Let us assume that
\begin{equation}
  \label{eq:positive-inner-product}
  \langle \Phi ( \widehat{A} - A_{T+1}, W-W_0 ) , \Phi ( \widehat{A} -
  A , \widehat{W}-W) \rangle > 0.
\end{equation}
Using H\"older's inequality, we obtain
\begin{align*}
  |\inr{U V^\top, \hat A - A}| &= |\inr{U V^\top, \cP_A(\hat A - A)}|
  \leq \norm{U V^\top}_\op \norm{\cP_A(\hat A - A)}_* =
  \norm{\cP_A(\hat A - A)}_*, \\
  |\inr{\Theta_A, \hat A - A}| &= |\inr{\Theta_A, \Theta_A \circ (\hat A
    - A)}| \leq \norm{\Theta_A}_\infty \norm{\Theta_A \circ (\hat A -
    A)}_1 = \norm{\Theta_A \circ (\hat A - A)}_1,
\end{align*}
and the same is done for $|\inr{\Theta_W, \hat W - W}| \leq
\norm{\Theta_W \circ (\hat W - W)}_1$. So,
when~\eqref{eq:positive-inner-product} holds, we obtain by rearranging
the terms of~\eqref{eq:thm2-main-decomposition}:
\begin{equation}
  \label{eq:proof-of-cone-constraint1}
  \begin{split}
    \tau \| \mathcal{P}_{A}^\perp ( \widehat{A} - &A)\|_* + \gamma\|
    \Theta_A^\perp \circ (\widehat{A} - A) \|_1 + \kappa\|
    \Theta_W^\perp \circ (\widehat{W} - W) \|_1 \\
    &\leq \tau \norm{\cP_A(\hat A - A)}_* + \gamma \norm{\Theta_A
      \circ
      (\hat A - A)}_1 + \kappa \norm{\Theta_W \circ (\hat W - W)}_1 \\
    &\quad + \frac{2}{d} \langle ( \frac{1}{\sqrt T} \bN_{T},
    -N_{T+1}) , \Phi( \widehat{A} - A , \widehat{W} - W ) \rangle.
  \end{split}
\end{equation}
Using Lemma~\ref{lem:noise-decomposition1}, together with H\"older's
inequality, we have for any $\alpha \in (0, 1)$:
\begin{equation}
  \label{eq:thm2-noise-control2}  
  \begin{split}
    \langle ( \frac{1}{\sqrt T} \bN_{T}, &-N_{T+1}) , \Phi(
    \widehat{A} - A , \widehat{W} - W ) \rangle = \inr{M, \hat A - A}
    + \frac{1}{T} \inr{\Xi, \hat W - W} \\
    &\leq \alpha \norm{M}_\op \norm{\cP_A(\hat A - A)}_* + \alpha
    \norm{M}_\op \norm{\cP_A^\perp(\hat A - A)}_* \\
    &\quad + (1 - \alpha) \norm{M}_\infty \norm{\Theta_A \circ (\hat A
      - A)}_1 + (1 - \alpha) \norm{M}_\infty \norm{\Theta_A^\perp
      \circ (\hat A - A)}_1 \\
    &\quad + \frac{1}{T} \norm{\Xi}_\infty ( \norm{\Theta_W \circ (\hat W - W)}_1 +  \norm{\Theta_W^\perp \circ (\hat W - W)}_1)~.
  \end{split}
\end{equation}

%In the same way, we have for any $\alpha \in (0, 1)$:
%\begin{equation}
 % \label{eq:thm2-noise-control3}  
 % \begin{split}
   % \langle ( \frac{1}{\sqrt T} \bN_{T}, &-N_{T+1}) , \Phi(
   % \widehat{A} - A, \widehat{W} - W ) \rangle  \\
   % &\leq \alpha \norm{M}_\op \norm{\hat A - A}_* + (1 - \alpha)
   % \norm{M}_\infty \norm{\hat A - A}_1 \\
   % &\quad + \frac 1T \norm{\Xi}_\infty \norm{\Theta_W \circ (\hat W -
    %  W)}_1 + \frac 1T \norm{\Xi}_\infty \norm{\Theta_W^\perp \circ
    %  (\hat W - W)}_1.
 % \end{split}
% \end{equation}

Now, using~\eqref{eq:proof-of-cone-constraint1}  together
with~\eqref{eq:thm2-noise-control2}, we obtain
\begin{align*}
  & \Big(\tau - \frac{2 \alpha}{d} \norm{M}_\op \Big)
  \norm{\cP_A^\perp(\hat A - A)}_* + \Big(\gamma - \frac{2 (1-\alpha)}{d}
  \norm{M}_\infty \Big) \norm{\Theta_A^\perp \circ (\hat A - A)}_1 \\
  & \quad \quad + \Big(\kappa - \frac{2}{d T} \norm{\Xi}_\infty \Big)
  \norm{ \Theta_W^\perp \circ( \hat W - W) }_1 \\
  &\leq \Big(\tau + \frac{2 \alpha}{d} \norm{M}_\op \Big)
  \norm{\cP_A(\hat A - A)}_* + \Big(\gamma + \frac{2 (1-\alpha)}{d}
  \norm{M}_\infty \Big) \norm{\Theta_A \circ (\hat A - A)}_1\\
  & \quad \quad + \Big(\kappa + \frac{2}{d T} \norm{\Xi}_\infty \Big)  \norm{ \Theta_W  \circ (\hat W - W)}_1
\end{align*}
which proves, using~\eqref{eq:thm-fast-rate-smoothing-parameters},
that
\begin{equation*}
  \tau \norm{\cP_A^\perp(\hat A - A)}_* + \gamma \norm{\Theta_A^\perp
    \circ (\hat A - A)}_1 \leq 5 \tau \norm{\cP_A(\hat A - A)}_* + 5
  \gamma \norm{\Theta_A \circ (\hat A - A)}_1.
\end{equation*}
This proves that $\hat A - A \in \cC_2(A, 5, \gamma / \tau)$. In the
same way, using~\eqref{eq:proof-of-cone-constraint1} with $A = \hat A$
together with~\eqref{eq:thm2-noise-control2}, we obtain that $\hat W -
W \in \cC_1(W, 5)$.

Now, using together~\eqref{eq:thm2-main-decomposition},
\eqref{eq:pythagoras} and~ \eqref{eq:thm2-noise-control2}
, and the fact that the
Cauchy-Schwarz inequality entails
\begin{align*}
  \norm{\cP_A(\hat A - A)}_* &\leq \sqrt{\rank{A}} \norm{\cP_A(\hat A
    - A)}_F, \quad |\inr{U V^\top, \hat A - A}| \leq \sqrt{\rank A}
  \norm{\cP_A(\hat A - A)}_F, \\
  \norm{\Theta_A \circ (\hat A - A)}_1 &\leq \sqrt{\norm{A}_0}
  \norm{\Theta_A \circ (\hat A - A)}_F, \quad |\inr{\Theta_A, \hat A -
    A}| \leq \sqrt{\norm{A}_0} \norm{\Theta_A \circ (\hat A - A)}_F~~.
\end{align*}
and similarly for $\hat W - W$, we arrive at
\begin{align*}
  \| \Phi &( \widehat{A} - A_{T+1}, \hat W - W_0 ) \|_2^2 + \| \Phi (
  \widehat{A} - A , \widehat{W}-W) \|_2^2 - \| \Phi ( A -
  A_{T+1}, W - W_0 ) \|_2^2 \\
  &\leq \Big(\frac{2 \alpha}{d} \norm{M}_\op + \tau \Big) \sqrt{\rank
    A} \norm{\cP_A(\hat A - A)}_F + \Big(\frac{2 \alpha}{d}
  \norm{M}_\op - \tau \Big) \norm{\cP_A^\perp(\hat A - A)}_* \\
  &\quad + \Big(\frac{2 \alpha}{d} \norm{M}_\infty + \gamma \Big)
  \sqrt{\norm{A}_0} \norm{\Theta_A \circ (\hat A - A)}_F + \Big(
  \frac{2 \alpha}{d} \norm{M}_\infty - \gamma\Big)
  \norm{\Theta_A^\perp \circ (\hat A - A)}_1 \\
  &\quad + \Big(\frac{2 \alpha}{d T} \norm{\Xi}_\infty + \kappa \Big)
  \sqrt{\norm{W}_0} \norm{\Theta_W \circ (\hat W - W)}_F + \Big(
  \frac{2 \alpha}{d T} \norm{\Xi}_\infty - \kappa \Big)
  \norm{\Theta_W^\perp \circ (\hat W - W)}_1,
\end{align*}
which leads, using~\eqref{eq:thm-fast-rate-smoothing-parameters}, to
\begin{align*}
  \frac 1d \| \Phi &( \widehat{A} - A_{T+1}, \hat W - W_0 ) \|_2^2 +
  \frac 1d \| \Phi ( \widehat{A} - A , \widehat{W}-W) \|_2^2 - \frac
  1d \| \Phi ( A -  A_{T+1}, W - W_0 ) \|_2^2 \\
  &\leq \frac{5\tau}{3} \sqrt{\rank A} \norm{\cP_A(\hat A - A)}_F +
  \frac{5\gamma}{3} \sqrt{\norm{A}_0} \norm{\Theta_A \circ (\hat A -
    A)}_F + \frac{5\kappa}{3} \sqrt{\norm{W}_0} \norm{\Theta_W \circ
    (\hat W - W)}_F.
\end{align*}
Since $\hat A - A \in \cC_2(A, 5, \gamma / \tau)$ and $\hat W - W \in
\cC_1(W, 5)$, we obtain using Assumption~\ref{ass:RE} and $ab \leq
(a^2 + b^2) / 2$:
\begin{align*}
  \frac 1d &\| \Phi ( \widehat{A} - A_{T+1}, \hat W - W_0 ) \|_2^2 +
  \frac 1d \| \Phi ( \widehat{A} - A , \widehat{W}-W) \|_2^2 \\
  &\leq \frac 1d \| \Phi ( A - A_{T+1}, W - W_0 ) \|_2^2 +
  \frac{25}{18} \mu_2(A, W)^2 \big( \rank(A) \tau^2 + \norm{A}_0
  \gamma^2) \\
  &\quad + \frac{25}{36} \mu_1(W)^2 \norm{W}_0 \kappa^2 + \frac 1d \|
  \Phi ( \widehat{A} - A , \widehat{W}-W) \|_2^2,
\end{align*}
which concludes the proof of Theorem~\ref{prop:fastrateRE}. $\hfill \square$

\subsection{Proof of Theorem~\ref{thm:take-away-message}}

For the proof of~\eqref{eq:explicit-upper-bound-W-prediction}, we
simply use the fact that $\frac{1}{d T} \norm{\bX_{T-1}(\hat W -
  W_0)}_F^2 \leq \EE(\hat A, \hat W)^2$ and use
Theorem~\ref{thm:convergence-rates}. Then we take $W = W_0$ in the
infimum over $A, W$.

For~\eqref{eq:explicit-upper-bound-W-l1}, we use the fact that since
$\hat W - W_0 \in \cC_1(W_0, 5)$, we have (see the Proof of
Theorem~\ref{prop:fastrateRE}), 
\begin{align*}
  \norm{\hat W - W_0}_1 &\leq 6 \sqrt{\norm{W_0}_0} \norm{\Theta_W
    \circ (\hat W - W_0)}_F \\
  &\leq 6 \sqrt{\norm{W_0}_0} \norm{\bX_{T-1} (\hat W - W_0)}_F /
  \sqrt{d T} \\
  &\leq 6 \sqrt{\norm{W_0}_0} \EE(\hat A, \hat W),
\end{align*}
and then use again Theorem~\ref{thm:convergence-rates}. The proof of
(\ref{eq:explicit-upper-bound-A-trace}) follows exactly the same
scheme. $\hfill \square$

\subsection{Concentration inequalities for the noise processes}
\label{sec:noise-controls}

The control of the noise terms $M$ and $\Xi$ is based on recent
developments on concentration inequalities for random matrices, see
for instance \cite{2010arXiv1004.4389T}. Moreover, the assumption on
the dynamics of the features's noise vector $(N_t)_{t \geq 0}$ is
quite general, since we only assumed that this process is a martingale
increment.  Therefore, our control of the noise $\Xi$ rely in
particular on martingale theory.

\begin{proposition}
  \label{prop:noise-controls}
  Under Assumption~\ref{ass:noise}, the following inequalities hold
  for any $x > 0$. We have
  \begin{equation}
    \label{eq:concentration_M_op}
    \bignorm{\frac 1d \sum_{j=1}^d (N_{T+1})_j \Omega_j}_\op \leq
    \sigma v_{\Omega, \op} \sqrt{\frac{2(x +
        \log(2n))}{d}}
  \end{equation}
  with a probability larger than $1 - e^{-x}$. We have
  \begin{equation}    
    \label{eq:concentration_M_infty}
    \bignorm{\frac 1d \sum_{j=1}^d (N_{T+1})_j \Omega_j}_\infty \leq
    \sigma v_{\Omega, \infty} \sqrt{\frac{2(x + 2 \log n)}{d}}    
  \end{equation}
  with a probability larger than $1 - 2 e^{-x}$, and finally
  \begin{equation}
    \label{eq:concentration_Xi_infty}
    \bignorm{\frac{1}{T+1} \sum_{t=1}^{T+1} \omega(A_{t-1})
      N_t^\top}_\infty \leq \sigma \sigma_\omega
    \sqrt{\frac{2 e(x + 2 \log d + \ell_T)}{T+1}}
  \end{equation}
  with a probability larger than $1 - 14 e^{-x}$, where
  \begin{equation*}
    \ell_T = 2 \max_{j = 1, \ldots, d} \log \log \bigg(
    \frac{\sum_{t=1}^{T+1} \omega_j(A_{t-1})^2}{T+1} \vee
    \frac{T+1}{\sum_{t=1}^{T+1} \omega_j(A_{t-1})^2} \vee e \bigg).
  \end{equation*}  
\end{proposition}

% \subsection{Proof of Proposition~\ref{prop:noise-controls} (Noise
%   control)}
% \label{sec:proof-noise-control}

\begin{proof}
  For the proofs of Inequalities~\eqref{eq:concentration_M_op}
  and~\eqref{eq:concentration_M_infty}, we use the fact that
  $(N_{T+1})_1, \ldots, (N_{T+1})_d$ are independent (scalar)
  subgaussian random variables.

  From Assumption~\ref{ass:noise}, we have for any $n \times n$
  deterministic self-adjoint matrices $X_j$ that $\E[ \exp( \lambda
  (N_{T+1})_j X_j) ] \mleq \exp( \sigma^2 \lambda^2 X_j^2 / 2 )$,
  where $\mleq$ stands for the semidefinite order on self-adjoint
  matrices. Using Corollary~3.7 from~\cite{2010arXiv1004.4389T}, this
  leads for any $x > 0$ to
  \begin{equation}
    \label{eq:tropp-result}
    \P\Big[ \lmax \Big( \sum_{j=1}^d (N_{T+1})_j X_j \Big) \geq x
    \Big] \leq n \exp\Big( -\frac{x^2}{2 v^2} \Big), \quad \text{where
    } v^2 = \sigma^2 \bignorm{\sum_{j=1}^d X_j^2}_{\op}.
  \end{equation}
  Then, following~\cite{2010arXiv1004.4389T}, we consider the dilation
  operator $\cL : \R^{n \times n} \rightarrow \R^{2n \times 2n}$ given
  by
  \begin{equation*}
    \cL(\Omega) =
    \begin{pmatrix}
      0 &\Omega \\
      \Omega^* &0
    \end{pmatrix}.
  \end{equation*}
  We have
  \begin{equation*}
    \Big\| \sum_{j=1}^d (N_{T+1})_j \Omega_j \Big\|_\op = \lmax\Big( \cL\Big(
    \sum_{j=1}^d (N_{T+1})_j \Omega_j \Big) \Big) = \lmax\Big(
    \sum_{j=1}^d (N_{T+1})_j \cL(\Omega_j) \Big)
  \end{equation*}
  and an easy computation gives
  \begin{equation*}
    \Big\| \sum_{j=1}^d \cL(\Omega_j)^2 \Big\|_\op = \Big\| \sum_{j=1}^d
    \Omega_j^\top \Omega_j \Big\|_\op \vee \Big\| \sum_{j=1}^d \Omega_j
    \Omega_j^\top \Big\|_\op.
  \end{equation*}
  So, using~\eqref{eq:tropp-result} with the self-adjoint $X_j =
  \cL(\Omega_j)$ gives
  \begin{equation*}
    \P\Big[ \bignorm{\sum_{j=1}^d (N_{T+1})_j \Omega_j}_\op \geq x
    \Big] \leq 2 n \exp\Big( -\frac{x^2}{2 v^2} \Big) \;\; \text{where
    } v^2 = \sigma^2 \bignorm{\sum_{j=1}^d \Omega_j^\top
      \Omega_j}_{\op} \vee \bignorm{\sum_{j=1}^d
      \Omega_j \Omega_j^\top}_{\op},
  \end{equation*}
  which leads easily to~\eqref{eq:concentration_M_op}.

  Inequality~\eqref{eq:concentration_M_infty} comes from the following
  standard bound on the sum of independent sub-gaussian random
  variables:
  \begin{equation*}
    \P \Big[ \Big| \frac 1d \sum_{j=1}^d (N_{T+1})_j (\Omega_j)_{k, l}
    \Big| \geq x \Big] \leq 2 \exp\Big(- \frac{x^2}{2 \sigma^2 (\Omega_j)_{k,
        l}^2} \Big)
  \end{equation*}
  together with an union bound on $1 \leq k, l \leq n$.

  Inequality~\eqref{eq:concentration_Xi_infty} is based on a classical
  martingale exponential argument together with a peeling argument. We
  denote by $\omega_j(A_{t})$ the coordinates of $\omega(A_{t}) \in
  \R^d$ and by $N_{t, k}$ those of $N_t$, so that
  \begin{equation*}
    \Big( \sum_{t=1}^{T+1} \omega(A_{t-1}) N_{t}^\top \Big)_{j, k} =
    \sum_{t=1}^{T+1} \omega_j(A_{t-1}) N_{t, k}.
  \end{equation*}
  We fix $j, k$ and denote for short $\eps_t = N_{t, k}$ and $x_t =
  \omega_j(A_t)$. Since $\E[\exp(\lambda \eps_t) | \cF_{t-1}] \leq
  e^{\sigma^2 \lambda^2 / 2}$ for any $\lambda \in \R$, we obtain by a
  recursive conditioning with respect to $\cF_{T-1}$, $\cF_{T-2},
  \ldots, \cF_0$, that
  \begin{equation*}
    \E \Big[ \exp \Big( \theta \sum_{t=1}^{T+1} \eps_t x_{t-1} -
    \frac{\sigma^2 \theta^2}{2} \sum_{t=1}^{T+1} x_{t-1}^2 \Big) \Big]
    \leq 1.
  \end{equation*}
  Hence, using Markov's inequality, we obtain for any $v > 0$:
  \begin{equation*}
    \P\Big[ \sum_{t=1}^{T+1} \eps_t x_{t-1} \geq x, \sum_{t=1}^{T+1}
    x_{t-1}^2 \leq v \Big] \leq \inf_{\theta > 0} \exp( -\theta x +
    \sigma^2 \theta^2 v / 2) = \exp\Big(-\frac{x^2}{2 \sigma^2 v} \Big),
  \end{equation*}
  that we rewrite in the following way:
  \begin{equation*}
    \P\Big[ \sum_{t=1}^{T+1} \eps_t x_{t-1} \geq \sigma \sqrt{2 v x},
    \sum_{t=1}^{T+1} x_{t-1}^2 \leq v \Big] \leq e^{-x}.
  \end{equation*}
  Let us denote for short $V_T = \sum_{t=1}^{T+1} x_{t-1}^2$ and $S_T
  = \sum_{t=1}^{T+1} \eps_t x_{t-1}$. We want to replace $v$ by $V_T$
  from the previous deviation inequality, and to remove the event $\{
  V_T \leq v \}$. To do so, we use a peeling argument. We take $v =
  T+1$ and introduce $v_k = v e^k$ so that the event $\{ V_T > v \}$
  is decomposed into the union of the disjoint sets $\{ v_{k} < V_T
  \leq v_{k+1} \}$. We introduce also $\ell_{T} = 2 \log \log\Big(
  \frac{\sum_{t=1}^{T+1} x_{t-1}^2}{T+1} \vee
  \frac{T+1}{\sum_{t=1}^{T+1} x_{t-1}^2} \vee e \Big)$.

  This leads to
  \begin{align*}
    \P\Big[ S_T \geq \sigma & \sqrt{2 e V_T (x + \ell_T)}, V_T > v
    \big] = \sum_{k \geq 0} \P\big[ S_T \geq \sigma \sqrt{2 e V_T (x +
      \ell_T)}, v_k < V_T \leq v_{k+1} \Big]�\\
    &= \sum_{k \geq 0} \P\Big[ S_T \geq \sigma \sqrt{2 v_{k+1} (x + 2
      \log \log(e^k \vee e))}, v_k < V_T \leq v_{k+1} \Big] \\
    &\leq e^{-x} (1 + \sum_{k \geq 1} k^{-2}) \leq 3.47 e^{-x}.
  \end{align*}
  On $\{ V_T \leq v \}$ the proof is the same: we decompose onto the
  disjoint sets $\{ v_{k+1} < V_T \leq v_{k} \}$ where this time $v_k
  = v e^{-k}$, and we arrive at
  \begin{equation*}
    \P\Big[ S_T \geq \sigma \sqrt{2 e V_T (x + \ell_T)}, V_T \leq v
    \big] \leq 3.47 e^{-x}.
  \end{equation*}
  This leads to
  \begin{equation*}
    \P \bigg[ \sum_{t=1}^{T+1} \omega_j(A_{t-1}) N_{t, k} \geq \sigma 
    \Big( 2 e \sum_{t=1}^{T+1} \omega_j(A_{t-1})^2 (x + \ell_{T, j})
    \Big)^{1/2} \bigg]
    \leq 7 e^{-x}
  \end{equation*}
  for any $1 \leq j, k \leq d$, where we introduced
  \begin{equation*}
    \ell_{T, j} = 2 \log \log\Big( \frac{\sum_{t=1}^{T+1}
      \omega_j(A_{t-1})^2}{T+1} \vee \frac{T+1}{\sum_{t=1}^{T+1}
      \omega_j(A_{t-1})^2} \vee e \Big).
  \end{equation*}
  The conclusion follows from an union bound on $1 \leq j, k \leq
  d$. This concludes the proof of
  Proposition~\ref{prop:noise-controls}.
\end{proof}

\bibliographystyle{plain}
\bibliography{graphlink6}

\begin{thebibliography}{10}

\bibitem{Abernethy09}
J.~Abernethy, F.~Bach, Th. Evgeniou, and J.-Ph. Vert.
\newblock A new approach to collaborative filtering: operator estimation with
  spectral regularization.
\newblock {\em JMLR}, 10:803--826, 2009.

\bibitem{Argyriou07}
A.~Argyriou, M.~Pontil, Ch. Micchelli, and Y.~Ying.
\newblock A spectral regularization framework for multi-task structure
  learning.
\newblock {\em Proceedings of Neural Information Processing Systems (NIPS)},
  2007.

\bibitem{Bickel09}
P.~J. Bickel, Y.~Ritov, and A.~B. Tsybakov.
\newblock Simultaneous analysis of lasso and dantzig selector.
\newblock {\em Annals of Statistics}, 37, 2009.

\bibitem{Breiman97}
L.~Breiman and J.~H. Friedman.
\newblock Predicting multivariate responses in multiple linear regression.
\newblock {\em Journal of the Royal Statistical Society (JRSS): Series B
  (Statistical Methodology)}, 59:3--54, 1997.

\bibitem{Candes09}
E.J. Cand{\`e}s and T.~Tao.
\newblock The power of convex relaxation: Near-optimal matrix completion.
\newblock {\em IEEE Transactions on Information Theory}, 56(5), 2009.

\bibitem{Candes05}
Cand{\`e}s E. and Tao T.
\newblock Decoding by linear programming.
\newblock In {\em Proceedings of the 46th Annual IEEE Symposium on Foundations
  of Computer Science (FOCS)}, 2005.

\bibitem{Evgeniou05}
Th. Evgeniou, Ch.~A. Micchelli, and M.~Pontil.
\newblock Learning multiple tasks with kernel methods.
\newblock {\em Journal of Machine Learning Research}, 6:615--637, 2005.

\bibitem{6034724}
S.~Gaiffas and G.~Lecue.
\newblock Sharp oracle inequalities for high-dimensional matrix prediction.
\newblock {\em Information Theory, IEEE Transactions on}, 57(10):6942 --6957,
  oct. 2011.

\bibitem{Kolar11}
M.~Kolar and E.~P. Xing.
\newblock On time varying undirected graphs.
\newblock {\em in Proceedings of the 14th International Conference on Artifical
  Intelligence and Statistics AISTATS}, 2011.

\bibitem{MR2555200}
V.~Koltchinskii.
\newblock The {D}antzig selector and sparsity oracle inequalities.
\newblock {\em Bernoulli}, 15(3):799--828, 2009.

\bibitem{MR2500227}
V.~Koltchinskii.
\newblock Sparsity in penalized empirical risk minimization.
\newblock {\em Ann. Inst. Henri Poincar\'e Probab. Stat.}, 45(1):7--57, 2009.

\bibitem{Koltchinskii11}
V.~Koltchinskii, K.~Lounici, and A.~Tsybakov.
\newblock Nuclear norm penalization and optimal rates for noisy matrix
  completion.
\newblock {\em Annals of Statistics}, 2011.

\bibitem{koren2008factorization}
Y.~Koren.
\newblock Factorization meets the neighborhood: a multifaceted collaborative
  filtering model.
\newblock In {\em Proceeding of the 14th ACM SIGKDD international conference on
  Knowledge discovery and data mining}, pages 426--434. ACM, 2008.

\bibitem{koren2010collaborative}
Y.~Koren.
\newblock Collaborative filtering with temporal dynamics.
\newblock {\em Communications of the ACM}, 53(4):89--97, 2010.

\bibitem{MR1363368}
A.~S. Lewis.
\newblock The convex analysis of unitarily invariant matrix functions.
\newblock {\em J. Convex Anal.}, 2(1-2):173--183, 1995.

\bibitem{liben2007link}
D.~Liben-Nowell and J.~Kleinberg.
\newblock The link-prediction problem for social networks.
\newblock {\em Journal of the American society for information science and
  technology}, 58(7):1019--1031, 2007.

\bibitem{Myers10}
S.A. Myers and Jure Leskovec.
\newblock On the convexity of latent social network inference.
\newblock In {\em NIPS}, 2010.

\bibitem{raguet2011generalized}
H.~Raguet, J.~Fadili, and G.~Peyr{\'e}.
\newblock Generalized forward-backward splitting.
\newblock {\em Arxiv preprint arXiv:1108.4404}, 2011.

\bibitem{Richard10}
E.~Richard, N.~Baskiotis, Th. Evgeniou, and N.~Vayatis.
\newblock Link discovery using graph feature tracking.
\newblock {\em Proceedings of Neural Information Processing Systems (NIPS)},
  2010.

\bibitem{Richard12}
E.~Richard, P.-A. Savalle, and N.~Vayatis.
\newblock Estimation of simultaneously sparse and low-rank matrices.
\newblock In {\em Proceeding of 29th Annual International Conference on Machine
  Learning}, 2012.

\bibitem{sarkar2010theoretical}
P.~Sarkar, D.~Chakrabarti, and A.W. Moore.
\newblock Theoretical justification of popular link prediction heuristics.
\newblock In {\em International Conference on Learning Theory (COLT)}, pages
  295--307, 2010.

\bibitem{Srebro05}
N.~{Srebro}, J.~D.~M. {Rennie}, and T.~S. {Jaakkola}.
\newblock Maximum-margin matrix factorization.
\newblock In Lawrence~K. Saul, Yair Weiss, and {L{\'e}on} Bottou, editors, {\em
  in Proceedings of Neural Information Processing Systems 17}, pages
  1329--1336. MIT Press, Cambridge, MA, 2005.

\bibitem{taskar2003link}
B.~Taskar, M.F. Wong, P.~Abbeel, and D.~Koller.
\newblock Link prediction in relational data.
\newblock In {\em Neural Information Processing Systems}, volume~15, 2003.

\bibitem{2010arXiv1004.4389T}
J.~A. {Tropp}.
\newblock {User-friendly tail bounds for sums of random matrices}.
\newblock {\em ArXiv e-prints}, April 2010.

\bibitem{Tsay05}
R.~S. Tsay.
\newblock {\em Analysis of Financial Time Series}.
\newblock Wiley-Interscience; 3rd edition, 2005.

\bibitem{MR2576316}
S.~A. van~de Geer and P.~B{\"u}hlmann.
\newblock On the conditions used to prove oracle results for the {L}asso.
\newblock {\em Electron. J. Stat.}, 3:1360--1392, 2009.

\bibitem{Vu11}
D.Q. Vu, A.~Asuncion, D.~Hunter, and P.~Smyth.
\newblock Continuous-time regression models for longitudinal networks.
\newblock In {\em Advances in Neural Information Processing Systems}. MIT
  Press, 2011.

\end{thebibliography}

\end{document}